\documentclass{article}

\usepackage[preprint]{nips_2018}

\usepackage[utf8]{inputenc} 
\usepackage[T1]{fontenc}    
\usepackage{hyperref}       
\usepackage{url}            
\usepackage{booktabs}       
\usepackage{amsfonts}       
\usepackage{nicefrac}       
\usepackage{microtype}      
\usepackage{graphicx}
\usepackage{psfrag}
\usepackage[usenames, dvipsnames]{color}
\usepackage{amsmath}
\usepackage{amsbsy}
\usepackage{color}
\usepackage{filecontents}
\usepackage[export]{adjustbox}
\usepackage{subcaption}

\usepackage{theorem}
\theoremheaderfont{\upshape\bfseries} \theoremstyle{plain}

\theorembodyfont{\rmfamily}

\theoremheaderfont{\itshape}

\newtheorem{property}{Property}

\newtheorem{theorem}{Theorem}
\newtheorem{proof}{Proof}

\newtheorem{comment}{Comment}

\newcommand{\comm}[1]{}

\DeclareMathOperator*{\argmax}{argmax}
\DeclareMathOperator*{\argmin}{argmin}

\title{Gradual Learning of Recurrent Neural Networks}

\author{
  Ziv Aharoni \\
  Department of Electrical Engineering\\
  Ben-Gurion University\\
  Beer-Sheva, Israel 8410501\\
  \texttt{zivah@post.bgu.ac.il} \\
   \And
   Gal Rattner \\
   Department of Electrical Engineering \\
   Ben-Gurion University\\
   Beer-Sheva, Israel 8410501\\
   \texttt{rattner@post.bgu.ac.il} \\
   \AND
   Haim Permuter \\
   Department of Electrical Engineering \\
   Ben-Gurion University\\
   Beer-Sheva, Israel 8410501\\
   \texttt{haimp@bgu.ac.il} \\
}

\begin{document}

\maketitle

\begin{abstract}
Recurrent Neural Networks (RNNs) achieve state-of-the-art results in many sequence-to-sequence modeling tasks. 
However, RNNs are difficult to train and tend to suffer from overfitting. 
Motivated by the Data Processing Inequality (DPI), we formulate the multi-layered network as a Markov chain, introducing a training method that comprises training the network gradually and using layer-wise gradient clipping.
We found that applying our methods, combined with previously introduced regularization and optimization methods, resulted in improvements in state-of-the-art architectures operating in language modeling tasks.

\end{abstract}

\section{Introduction}
Several forms of Recurrent Neural Network (RNN) architecture, such as LSTM \cite{HOCHREITER} and GRU \cite{cho2014enc}, have achieved state-of-the-art results in many sequential classification tasks \cite{CHO,Hypernetworks,Resnets,Dropin,InitAndMomentum,RHN} during the past few years. 
The number of stacked RNN layers, i.e. the network depth, has key importance in extending the ability of the architecture to express more complex dynamic systems \cite{Bianchini, Montufar}. However, training deeper networks poses problems that are yet to be solved.
\par Training a deep RNN network to exploit its performance potential can be a very difficult task. 
One of the problems in training deep networks (not recurrent necessarily) is the degradation problem, as studied in \cite{Resnets}. 
Moreover, RNNs are exposed to exponential vanishing or exploding gradients through the Back-Propagation-Through-Time (BPTT) algorithm. 
Many studies have attempted to address those problems by regularizing the network \cite{RNNBatchNorm,VariationalDropout,Zaremba}, using different layer initialization methods \cite{GreedyLayerwiseInit, hinton2006fast}, or using shortcut connections between layers \cite{Resnets,Dropin,RHN}.

 An additional problem in training a deep architecture is the \textit{covariate shift}. Previous studies \cite{RNNBatchNorm, IofeSzegedy, Shimodaira} have shown that the \textit{covariate shift} has a negative effect on the training process among deep neural architectures. Covariate shift is the change in a layer's input distribution during training, also manifested as \textit{internal covariate shift}, when a network internal layer input distribution is changed due to a shallower layer training process.

\par In this paper, we revisit the constructive/incremental approach that breaks the optimization process into several learning phases. Each learning phase includes training an increasingly deeper architecture than the previous ones. 
We show that given a specific architecture is capable of approximating a specific function/dynamic system, one can train it gradually to obtain the same performance.
In this way, one can train the network gradually, reducing the deleterious effects of degradation and backpropagation problems. Additionally, we suggest that by adjusting the gradient clipping norms in a layerwise manner, we are able to improve the network performance even further by reducing the covariate shift.

In order to evaluate our method's performance, we conducted experiments over the word-level Penn Treebank (PTB) and the Wikitext-2 datasets, which are commonly used in evaluating language modeling tasks from the field of natural language processing. 
We demonstrated that combining the Gradual Learning (GL) method and layerwise regularization adjustments can significantly improve the results over the traditional training method of LSTM.  

\section{Related Research}

In the past years, the idea of incremental/constructive learning has been explored in many various previous works.
\citet{Dropin} introduced the method of gradually adding layers to the network by a \textit{Dropin} layer. 
\citet{net2net} introduced a method for expanding a feedforward network using a function preserving transformation. 
\citet{TrainingAnalysingDRNN} studied training deep RNNs and their ability to process the data at multiple timescales. 
They proposed two configurations, one of which (DRNN-1O) comprises the same training scheme as that we present in our work. 
The main difference from the preceding works is that we do not intend to find the best incremental training scheme for training a RNN (or DNN). Instead, we seek to claim that using these methods could yield optimal models, and due to the difficulties of training deep networks, using these methods could yield improved results.

\citet{GreedyLayerwiseInit} proposed that a proper initialization of the network parameters by a layer-by-layer greedy unsupervised initialization process could recognize features in the input structure that would be valuable to the classifier at the network's output. This process is done by a reconstruction objective that encourages the network states to preserve all the information of the input. This constitutes the main difference from our method that seeks to preserve only the information of the inputs that is relevant for estimating the labels.

\citet{IofeSzegedy} and \citet{RNNBatchNorm} introduced and addressed the covariate shift problem in deep architectures, using batch normalization as a straightforward solution by fixing the layer activation distribution. 
While this method may be effective in overcoming the covariate shift phenomena and claimed to accelerate the training process, it poses restrictions such as large batch size. 
Our work approaches this issue from a different perspective and suggests a way to reduce the covariate shift by clipping the norm of the update step per layer while training the network gradually. 
Our proposed methods avoid the addition of excessive normalization layers and allow training using smaller batches, such that the training process is slower yet ends up with improved performance.

\section{Notation}
Let us represent a network with $l$ layers as a mapping from an input sequence $X \in \mathcal{X}$ to an output sequence $\hat{Y}_l \in \mathcal{Y}$ by 
\begin{equation}
\hat{Y}_l =  S_l \circ f_{l} \circ f_{l-1} \circ \dots \circ f_1 (X ; \Theta_l), \label{net_transform}\\
\end{equation}
where $f_2 \circ f_1$ represents the composition of $f_2$ over $f_1$. 
The term $\Theta_l = \{\theta_1, \dots, \theta_{l}, \theta_{S_l}\}$ denotes the network parameters, such that $\theta_k$ is the parameters of the $\text{k}^\text{th}$ layer, that is $f_k$. The term $\theta_{S_l}$ denotes the parameters of the softmax mapping $S_l$ that is applied to a network with $l$ layers. 
For convenience, we denote $F_l = f_l \circ f_{l-1} \circ \dots \circ f_1$.
We also define the $\text{l}^\text{th}$ layer \textit{state sequence} by $T_l =  F_{l} (X ; \theta^l)$,
where $\theta^l = \{\theta_1, \dots, \theta_{l}\}$, and we define the $\text{l}^\text{th}$ layer cost function by $J(\Theta_l) =  \mathrm{cost} (\hat{Y}_l, Y)$.
Next, we define the gradient vector with respect to $J(\Theta)$ by $\mathbf{g}= \frac{\partial}{\partial \Theta} J(\Theta)$, and the gradient vector of the $\text{k}^\text{th}$ layer parameters with respect to $J(\Theta)$ by $\mathbf{g}_{k} = \frac{\partial}{\partial \theta_k} J(\Theta)$.

\section{Gradual Learning}
In this section, we discuss a theoretical motivation that explains why a greedy training scheme is a reasonable approach for training  a deep  neural network. 
Then, we elaborate on the implementation of the method.

\subsection{Theoretical motivation}

The structure of a neural network comprises a sequential processing scheme of its inputs. This structure constitutes the Markov chain
\begin{equation}
    Y-X-T_1-T_2-\dots-T_L \label{nn-markov}.
\end{equation} 
This Markov chain has an elementary and well-known property that is used in our analysis, hence it is given without a proof.
\begin{property} \label{prop_markov}
Given a Markov chain $A_1-A_2-\dots-A_N$, for any ordered triplet $i,j,k \in \{1,\dots, N\}$, such that $i<j<k$, the Markov chain $A_i-A_j-A_k$ holds.
\end{property}
Hence, for every $T_l$ we can consider the Markov chain $Y-X-T_l$ that induces the conditional probabilities $p(y|x)$ and $p(y|t_l)$. 
Note that within the scope of training a neural network, $p(y,x,t_l)$ can be factorized by
\begin{align}
    p(y,x,t_l)  &= p(y|x)p(x)q_{\Theta_l}(t_l|x),
\end{align} 
where the terms $p(y|x), p(x)$ are induced entirely from the underlying distribution that generated the data, namely $P_{X,Y}$. Yet, the term $q_{\Theta_l}(t|x)$ is determined by the network parameters. Hence, any modification of the joint distribution $p(y,x,t_l)$ could be achieved only by modifying $q_{\Theta_l}(t|x)$. In particular, this means that $p(y|t_l)$ is also affected by the network through the term $q_{\Theta_l}(t|x)$, as shown below: 
\begin{align}
    p(y|t_l)  &=\frac{p(y,t_l)}{p(t_l)} \nonumber\\
            &=\frac{\sum_x p(x,y)q_{\Theta_l}(t_l|x)}{\sum_{x^\prime} p(x^\prime)q_{\Theta_l}(t_l|x^\prime)}.
\end{align} 
Next, we estimate $p(y|t_l)$ by the Maximum Likelihood Estimator (MLE) or, equivalently, by the negative log loss function.
Given a training set of $N$ examples $S = \left\{ (x_i,y_i)\right\}_{i=1}^N$ drawn i.i.d from an unknown distribution $P_{X,Y} = P_X P_{Y|X}$, we want to estimate $P_{Y|T_L}$ by a $L$-layered neural network that is parameterized by $\Theta_L$. Let us denote the estimator by $Q_{Y|T_L}^\Theta$, where $T_L = F_L(X)$. Thus, the MLE is given by
\begin{align}
    \Theta_L^\ast &= \argmax_\Theta \frac{1}{N}\sum_{i=1}^N \log \left[Q_{Y|T_L}^\Theta(y_i|t^L_i)\right] \\
            &= \argmin_\theta -\frac{1}{N}\sum_{i=1}^N \log \left[Q_{Y|T_L}^\Theta(y_i|t^L_i)\right], \label{mle_arg}
\end{align}
where $t^L_i$ denotes $F_L(x_i)$. Next, we optimize the maximum likelihood criteria. 
The following theorem is well-known, but for completeness we present a proof which is helpful for understanding the subsequent arguments.
\begin{theorem}[MLE and minimal negative log-likelihood] \label{opt_ll}
Given a training set of $N$ examples $S = \left\{ (x_i,y_i)\right\}_{i=1}^N$ drawn i.i.d from an unknown distribution $P_{X,Y} = P_X P_{Y|X}$, the MLE is given by $P_{Y|X}$ and the optimal value of the  criteria is $H(Y|X)$.
\end{theorem}
\begin{proof}
By the law of large numbers, the maximum likelihood criteria in \eqref{mle_arg} converges to
\begin{align}
    -\frac{1}{N}\sum_{i=1}^N \log \left[Q_{Y|T_L}^\Theta(y_i|t^L_i)\right] &\stackrel{N\rightarrow \infty}{\longrightarrow}   \mathbb{E}_{P_{X,Y}}\left[ -\log Q_{Y|T_L}^\Theta(Y|T^L) \right] \\
                   &=\mathbb{E}_{P_{X,Y}}\left[ -\log P_{Y|X}(Y|X) \right] + \mathbb{E}_{P_{X,Y}}\left[ \log \frac{P_{Y|X}(Y|X)}{Q_{Y|T_L}^\Theta(Y|T^L)} \right]\\
                   &=H(Y|X) + D_{KL}(P_{Y|X}\lVert Q_{Y|T_L}^\Theta) \label{mle},
\end{align}
where, $H(Y|X)$ denotes the conditional entropy of $Y$ given $X$ and $D_{KL}(P_{Y|X}\lVert Q_{Y|T_L}^\Theta)$ denotes the Kullback–Leibler (KL) divergence between $P_{Y|X}$ and $Q_{Y|T_L}^\Theta$. Due the non-negativity of the KL divergence, and since only $Q_{Y|T_L}^\Theta$ depends on $\Theta$, the negative log likelihood is minimized when $D_{KL}(P_{Y|X}\lVert Q_{Y|T_L}^\Theta)=0$, which happens if and only if $Q_{Y|T_L}^\Theta = P_{Y|X}$. 
Hence, the MLE is $P_{Y|X}$ and the optimal bound for the negative log-likelihood loss is $H(Y|X)$. $\quad \square$
\end{proof}
Thus, under the optimality conditions of Theorem \eqref{opt_ll} we obtain that the estimate for $p(y|t_L)$, that is $Q_{Y|T_L}^\Theta$, satisfies $p(y|t_L) = p(y|x)$.
We proceed by using the data processing inequality, which is is given as follows.
\begin{theorem}[Data-processing inequality {\cite[section 2.8]{cover2012elements}}]
For $X, Y, Z$ random variables, if the Markov relation $X- Y- Z$ holds, then $I(X;Y)\geq I(X;Z)$, where $I(X;Y)$ denotes the mutual-information between $X$ and $Y$.
\end{theorem}

By the Data Processing Inequality (DPI) and Property \eqref{prop_markov}, we can claim that for all $l=1,\dots,L$ the following statement hold
\begin{align}
    I(Y;X) &\geq I(Y;T_l) \label{transform_info}. 
\end{align}
This means that by processing the inputs $X$, one cannot increase the information between $X$ and $Y$.

Next, we show that under the conditions of Theorem \eqref{opt_ll}, Equation \eqref{transform_info} holds with equality.
\begin{theorem} \label{thm:gl}
    If $Q_{Y|T_L}^\Theta$ satisfies the optimality conditions of Theorem \eqref{opt_ll}, then $I(X;Y) = I(T_l;Y),\quad \forall l=1,\dots,L$.
\end{theorem}
\begin{proof}
By the Markov relation of the network \eqref{nn-markov} and Property \eqref{prop_markov} we can say that
\begin{equation}
    Y-X-T_L,\label{mark:y-x-t}
\end{equation}
and hence, by definition,
\begin{equation}
    p(y|x,t_L) = p(y|x). \label{y-x-t}
\end{equation} 
By the optimatilty condition of Theorem \eqref{opt_ll} we can say that 
\begin{equation}
    p(y|x) = p(y|t_L). \label{y|x=y|t}
\end{equation} 
Combining \eqref{y-x-t} and \eqref{y|x=y|t} we get
\begin{equation}
    p(y|x,t_L) = p(y|t_L), \label{y-t-x}
\end{equation} 
which shows by definition that the following Markov chain holds:
\begin{equation} \label{mark:y-t-x}
    Y-T_L-X.
\end{equation} 
According to the Markov relations \eqref{mark:y-x-t} and \eqref{mark:y-t-x} the following hold:
\begin{align}
    p(y,t_L|x) &= p(y|x)p(t_L|x) \\
    p(y,x|t_L) &= p(y|t_L)p(x|t_L). 
\end{align}
Therefore, by definition, $I(T_L;Y|X)=0$ and $I(X;Y|T_L)=0$. Hence, 
\begin{align}
    I(X,T_L;Y)  &= I(X;Y) + I(T_L;Y|X) \\
                &= I(T_L;Y) + I(X;Y|T_L),
\end{align}
where the equalities follow the chain rule for mutual information as given in \cite[Section~2.5.2]{cover2012elements}. 
Thus we concluded that $I(X;Y)=I(T_L;Y)$. According to the DPI we can claim directly that  $I(X;Y)=I(T_l;Y), \forall l=1,\dots,L$, as desired. $\quad \square$
\end{proof}

We showed that a necessary condition to achieve the MLE is that the network states, namely $\left\{T_l\right\}_{l=1}^L$, will satisfy $I(Y;X)=I(Y;T_l)$. 
This means that an optimal model with $L$ layers will satisfy that all the relevant information about $Y$ within $X$ must pass through every layer of the network, in particular layer 1. 
Due to the fact that shallow networks are easier to train, and since minimizing cross-entropy "drives" the network states to contain all the relevant information about $Y$ (Theorem \eqref{thm:gl}), we propose a greedy training scheme. 
In this way we can overcome the difficulties of training a deep architecture and simultaneously assure that an optimal model could be achieved.

\begin{comment}
    We cannot guarantee that after training a layer, say layer $l_0$, the negative log likelihood is minimized, either because of the model limitation or because the optimization yielded a local minimum and not a global one. 
    Hence, we cannot claim  that we maximized the mutual information between $Y$ and $T_{l_0}$. 
    But since this is a necessary condition to achieve the optimal condition of Theorem \eqref{opt_ll} (and not a sufficient condition), we assume that by the end of training, when the negative log-likelihood is small, $I(Y;T_{l_0})$ is maximized (or relatively close to it).
\end{comment}

\begin{figure}[!ht]
  \centering
  \psfrag{A}[c][][.65]{layer 1}
  \psfrag{B}[c][][.65]{layer 2}
  \psfrag{C}[c][][.65]{layer 3}
  \psfrag{D}[c][][.65]{softmax 1}
  \psfrag{E}[c][][.65]{softmax 2}             
  \psfrag{F}[c][][.65]{softmax 3}             
  \psfrag{G}[c][][.65]{cost 1}              
  \psfrag{H}[c][][.65]{cost 2}    
  \psfrag{I}[c][][.65]{cost 3}             
  \psfrag{a}[c][][.65]{$\textbf{x}$}    
  \psfrag{b}[c][][.65]{$\textbf{y}$}      
  \psfrag{c}[c][][.65]{}        
  \psfrag{d}[c][][.65]{}        
  \psfrag{e}[c][][.65]{}        
  \psfrag{f}[c][][.65]{}        
  \psfrag{g}[c][][.65]{}        
  \psfrag{h}[c][][.65]{}        
  \psfrag{i}[c][][.65]{initialize}        
  \psfrag{x}[c][][.8]{phase 1}          
  \psfrag{y}[c][][.8]{phase 2}          
  \psfrag{z}[c][][.8]{phase 3}                
  \includegraphics[scale=0.4]{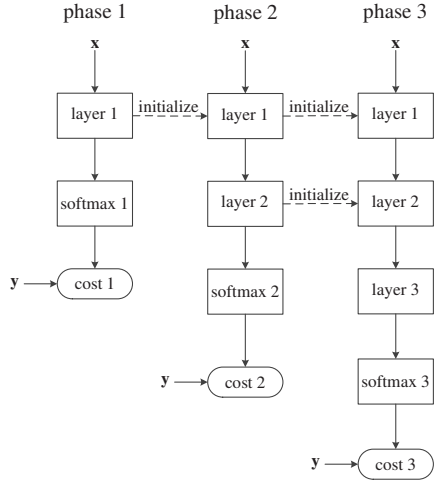}
  \caption{\textbf{Depiction of our training scheme for a 3 layered network.} At phase 1 we optimize the parameters of layer 1 according to cost 1. At phase 2, we add layer 2 to the network, and then we optimize the parameters of layers 1,2, when layer 1 is copied from phase 1 and layer 2 is initialized randomly.  At phase 3, we add layer 3 to the network, and then we optimize all of the network's parameters, when layers 1,2 are copied from phase 2 and layer 3 is initialized randomly.}
  \label{learning_scheme}
\end{figure}

\subsection{Implementation}

Now, motivated by the conclusions of the preceding section, we propose to break up the optimization process into $L$ phases, as the number of layers, optimizing $J(\Theta_l)$ sequentially as $l$ increases from 1 to L.
In each phase $l$, the network is optimized with respect to $\hat{Y}_l$ until convergence. 
Once the training of phase $l$ is done, phase $l+1$ is initiated with layers $1,\dots,l$ initialized with weights $\theta^l$ learned in the previous phase, and $\theta_{l+1}$ is initialized randomly. Initialization of the softmax layer's weights $\theta_{S_{l+1}}$ can be done either randomly or by inheritance of $\theta_{S_{l}}$ from the preceding training phase. An example for a training scheme is depicted in Figure \ref{learning_scheme}.

\section{Layer-wise Regularization Adjustments}
Driven by the Markov chain realtion of the neural network and the DPI, we conclude that exploiting the training procedure at any of the training phases will be most beneficial. Adjusting hyper-parameters for each training phase separately, such that it improves the minimization of the loss function is a necessary step to ensure the quality of our method implementation. However, adjusting the hyper-parameters might increase the robustness of the experiments and add variance to the experiments results as suggested in \cite{melis2017state}.

\subsection{Layer-wise Gradient Clipping (LWGC)} \label{lwgc}
Performing gradient clipping over the norm of the entire gradient vector $\mathbf{g}$, as suggested in \cite{pascanu}, may cause a contraction of the relatively small elements in the gradient vector on each update step, due to presence of much larger gradient elements in the vector that are much more dominant in the squared elements sum of the global norm $\left\Vert\mathbf{g}\right\Vert$. Due to this phenomenon, global gradient clipping, although proven to be useful and widely used in many architectures \cite{Regularizing_Optimizing_LSTM_LM, mos,Zaremba}, may have negative impact on the update step size for some weights elements in the network.
Clipping the gradients of each layer separately will eliminate the influence of gradient contraction between the different layers of the network.
 Global gradient norm clipping of vector $\mathbf{g}$ is formulated as $\hat{\mathbf{g}} := \frac{\mu}{\max(\mu,\left\Vert\mathbf{g}\right\Vert)}\mathbf{g}$, where $\mu$ is the fixed maximum gradient norm hyper-parameter.
  

\begin{figure}[!ht]
  \centering
  \includegraphics[scale=0.35, center]{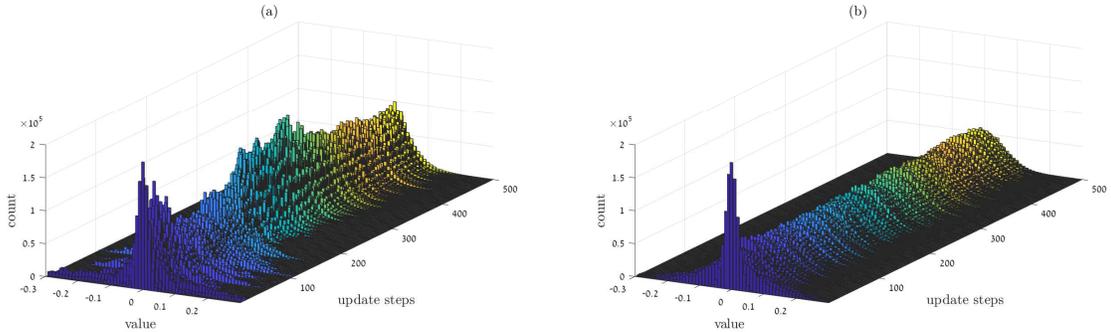}
  \caption{Depiction of the \textit{covariate shift} reduction at the initial stage of the training process, offering a comparison of the histograms of the third layer activations during the first 500 updates on training. Figure (a) shows the histogram of an AWD-MoS-LSTM network on traditional training. Figure (b) shows the histogram of the AWD-MoS-LSTM network trained using GL-LWGC.}
  \label{covariate-shift}
\end{figure}

When training layers gradually, at the beginning of each training phase the randomly initialized newly added layer tends to have a significantly larger gradient norm compared to the shallower pre-trained layers. Considering the differences in the layers' gradient norms, we suggest that treating each layer weights' gradient vector individually and clipping the gradients vector layerwise can reduce internal covariate shift significantly, as depicted in Figure \ref{covariate-shift}. 
In our experiments, we clipped each layer's gradient separately, increasing the clipping norm as the layer is deeper in the network. Moreover, we used a strictly low gradient clipping norm on the encoder matrix to restrict the entire network's covariate shift during training.

The formulation of the LWGC method for a network with $N$ layers is given by 
\begin{equation}
    \left[ \hat{\mathbf{g}}_1^T, \dots, \hat{\mathbf{g}}_L^T \right]^T := \left[ \frac{\mu_1}{\max(\mu_1,\left\Vert\mathbf{g}_1\right\Vert)}\mathbf{g}_1^T, \dots, \frac{\mu_N}{\max(\mu_N,\left\Vert\mathbf{g}_N\right\Vert)}\mathbf{g}_N^T \right]^T,
\end{equation} where $\mathbf{g}_i$ is the gradient vector w.r.t the weights of layer $i$, and $\mu_i$ is a fixed maximum gradient hyper-parameter of layer $i$.

\section{Experiments}
We present results on two datasets from the field of natural language processing, the word-level Penn Tree-Bank (PTB) and Wikitext-2 (WT2). We conducted most of the experiments on the PTB dataset and used the best configuration (except for minor modifications) to evaluate the WT2 dataset.

Following previous work \cite{melis2017state}, we established our experiments based on the pytorch implementation of \citet{mos}.
We have added our methods (GL, LWGC) to the implementation, closely following the hyper-parameter settings, for fair comparison. 
This implementation includes variational dropout as proposed by \cite{VariationalDropout}, Weight Tying (WT) method as proposed in \cite{WT}, the regularization and optimization techniques as proposed in \cite{Regularizing_Optimizing_LSTM_LM} and the mixture of softmaxes as proposed by \cite{mos}. 
Dynamic evaluation \cite{dynamicEval} was applied on the trained model to evaluate the performance of our model compared to previous state-of-the-art results.

We conducted two models in our experiments, a \textit{reference} model and a \textit{GL-LWGC LSTM} model that was used to check the performance of our methods. The \textit{reference} model is 3 layered LSTM optimized and regularized with the properties described in \cite{mos}\footnote[3]{https://github.com/zihangdai/mos} with a slight change in the form of enlarging the third layer from 620 to 960 cells, in order to even the total number of parameters. The model was trained for 1000 epochs until the validation score stopped improving. Our reference model failed to improve the previous results presented by \cite{mos}. 

\begin{table}[!ht] 
  \caption{Single model test perplexity of the PTB dataset}
  \centering
  \begin{tabular}{lllll}
    \toprule
    \textbf{Model}	& \textbf{Size} 	& \textbf{Valid} & \textbf{Test} \\
    \toprule
    
    \citet{RHN} - Variational RHN + WT 		 						    & 23M & 67.9 & 65.4 \\
    \citet{NAS} - NAS & 25M & - & 64.0 \\
    \citet{melis2017state} - 2-layer skip connection LSTM & 24M &  60.9 & 58.3 \\
    \citet{Regularizing_Optimizing_LSTM_LM} - AWD-LSTM                  & 24M & 60.0 & 57.3 \\
    \citet{mos} - AWD-LSTM-MoS & 22M & 58.08 & 55.97\\
    \citet{mos} - AWD-LSTM-MoS + finetune & 22M & 56.54 & 54.44\\
    \midrule
    Ours - Reference Model         & 26M &  57.41       & 55.58 \\
    Ours - 2-layers GL-LWGC-AWD-MoS-LSTM + finetune & 19M & 55.18 & 53.54\\
    Ours - GL-LWGC-AWD-MoS-LSTM & 26M & 54.57 & 52.95 \\
    Ours - GL-LWGC-AWD-MoS-LSTM + finetune & 26M & \textbf{54.24} & \textbf{52.57} \\
    \midrule
    \citet{dynamicEval} AWD-LSTM + dynamic evaluation\footnote[2] & 24M & 51.6 & 51.1 \\
    \citet{mos} AWD-LSTM-MoS + dynamic evaluation\footnote[2] & 22M & 48.33 & 47.69 \\
    Ours - GL-LWGC-AWD-MoS-LSTM + dynamic evaluation\footnote[2] & 26M & \textbf{46.64} & \textbf{46.34} \\
    \bottomrule
  \end{tabular}
  \label{res_table}
\end{table}

\begin{table}[!ht] 
  \caption{Single model perplexity of the WikiText-2 dataset}
  \centering
  \begin{tabular}{lllll}
    \toprule
    \textbf{Model}	& \textbf{Size} 	& \textbf{Valid} & \textbf{Test} \\
    \toprule
    \citet{melis2017state} - 2-layer skip connection LSTM & 24M &  69.1 & 65.9 \\
    \citet{Regularizing_Optimizing_LSTM_LM} - AWD-LSTM + finetune                 & 33M & 68.6 & 65.8 \\
    \citet{mos} - AWD-LSTM-MoS + finetune & 35M & 63.88 & 61.45\\
    \midrule
    Ours - GL-LWGC-AWD-MoS-LSTM & 35M & 63.59 & 61.27 \\
    Ours - GL-LWGC-AWD-MoS-LSTM + finetune & 38M & \textbf{62.79} & \textbf{60.54} \\
    \midrule
    \citet{dynamicEval} AWD-LSTM + dynamic evaluation\footnote[2] & 33M & 46.4 & 44.3 \\
    \citet{mos} AWD-LSTM-MoS + dynamic evaluation\footnote[2] & 35M & 42.41 & 40.68 \\
    Ours - GL-LWGC-AWD-MoS-LSTM + dynamic evaluation\footnote[2] & 38M & \textbf{42.19} & \textbf{40.46} \\
    \bottomrule
  \end{tabular}
  \label{wiki_res_table}
\end{table}
In our \textit{GL-LSTM} model we applied  the GL method with 3 training phases, at each we added a single layer of 960 cells and applied LWGC with an increasing maximum gradient norm at deeper layers. The LWGC max gradient norm for the LSTM and softmax layers was set in the range of 0.12-0.17, and for the embedding matrix a maximum gradient norm of 0.035-0.05 was set in order to lower the covariate shift along the training process. Other than that, regularization methods and hyper-parameter settings similar to those of the \textit{reference} model were used. Our GL-LSTM model overcame the state-of-the-art results with only two layers and 19M parameters, and further improved the state-of-the-art results with the third layer phase. Results of the \textit{reference} model and \textit{GL-LWGC LSTM} model are shown in Table \ref{res_table}. 

Experiments on the WT2 database were conducted with the same parameters as were used with the PTB model except for enlarging the embedding size to 300 units and changing the LSTM hidden size to 1050 units. Results of the WT2 model are shown in Table \ref{wiki_res_table}.

\section{Ablation Analysis}
In order to evaluate the benefit gained by each of our proposed methods, we measured the performance of our best-performing model removing one of the methods each time. Other than that, in order to provide a fair comparison with the previously suggested state-of-the-art methods \cite{mos} we evaluated a reference AWD-LSTM-MoS model with an enlarged third layer, to make the comparison with the same number of parameters as in our models. Table \ref{ablation_table} shows validation and test results for the ablated models.

\footnotetext[2]{Marking dynamic eval methods that update the model while evaluating, to distinguish from static evaluation models.}

\begin{table}[!ht] 
  \label{table:ablation}
  \centering
  \begin{tabular}{llll}
  \toprule
    & \quad \quad \quad PTB &  \\
    \textbf{Model} 	& \textbf{Valid} & \textbf{Test} \\
    \midrule
    \midrule
    GL-LWGC-AWD-MoS-LSTM    &  54.24     & 52.57 \\
    \toprule
    w/o fine-tuning            &  54.57    & 52.95 \\
    w/o LWGC                  &  56.32    & 54.09 \\
    w/o GL                    &  55.43    & 53.70 \\

    \bottomrule
    \hfill
  \end{tabular}
  \caption{Ablation analysis of the best-performing LSTM models over the Penn Treebank dataset. All models with 26M parameters.}
  \label{ablation_table}
\end{table}

Performance measurement of the GL ablated model, required setting the hyper-parameters ahead of initializing the training process. In this case we set the hyper-parameters as in the last phase of training of our best-performing GL-LWGC model, yet allowing a larger number of epochs to exploit the convergence. The ablated LWGC model was trained in three phases, each equivalent in length and hyper-parameter settings to the parallel phase in the best-performing GL-LWGC model, except for the global gradient norm clipping. While training the ablated LWGC model we set the maximum global gradient norm to be 
   $\mu_{global,l} = \sqrt{\mu_{L_1,l}^2 + \mu_{L_2,l}^2 + \dots + \mu_{L_N,l}^2}$,  
where $\mu_{L_i,l}$ is the maximum gradient norm for element $i$ in phase $l$ using LWGC, and $\mu_{global,l}$ is the maximum gradient norm for the non-LWGC case in phase $l$.

Each of the ablated GL and ablated LWGC models outperformed the previous state-of-the-art results, Yet combining both of the methods showed improved results.
The ablation analysis shows that the LWGC method has a stronger impact on the final results, while the GL method reduces the effectiveness of premature fine-tuning. Applying fine-tuning on the ablated LWGC model was not effective.  

We tested several hyper-parameters settings for our reference model, yet failed to improve the results presented by \cite{mos}. This result empowers our conclusions that the GL method decreases the degradation problem caused by increasing the depth or layer size of a model. 

\section{Conclusions}
We presented an effective technique to train RNNs. GL increases the network depth gradually as training progresses, and LWGC adjusts a gradient clipping norm in a layerwise manner at every learning phase of the training. Our techniques are implemented easily and do not involve increasing the amount of parameters of a network. We demonstrated the effectiveness of our techniques on the PTB and WikiText-2 datasets. We believe that our techniques would be useful for additional neural network architectures, such as GRUs and stacked RHNs, and in additional settings, such as in reinforcement learning.

\medskip

\small

\bibliographystyle{plainnat}
\bibliography{main}

\end{document}